\documentclass{article}

\usepackage[nonatbib,preprint]{neurips_template} 

\usepackage[utf8]{inputenc} 
\usepackage[T1]{fontenc}    
\usepackage{hyperref}       
\usepackage{url}            
\usepackage{booktabs}       
\usepackage{amsfonts}       
\usepackage{nicefrac}       
\usepackage{microtype}      
\usepackage{xcolor}         
\usepackage{multirow}
\usepackage{tikz}
\usetikzlibrary{matrix, positioning}
\usepackage{amssymb}
\usepackage{amsmath}
\usepackage{amsthm}
\usepackage{math_commands}
\usepackage{comment}
\usepackage{colortbl}

\newcounter{ToDo}
\newcounter{gaocomm}

\newcounter{Note1}
\definecolor{blue-violet}{rgb}{0.00,0.75,0.90}
\definecolor{mygreen}{rgb}{0.0, 0.5, 0.0}
\definecolor{awesome}{rgb}{1.0, 0.13, 0.32}
\definecolor{bostonuniversityred}{rgb}{1.0, 0.0, 0.0}

\title{A Simple Yet Effective SVD-GCN for Directed Graphs}

\author{%
  Chunya Zou, Andi Han, Lequan Lin and Junbin Gao \\
  Discipline of Business Analytics\\
  The University of Sydney Business School\\
  The University of Sydney\\
  Camperdown, Sydney, NSW 2006, Australia \\
  \texttt{\{czou6634,llin0615\}@uni.sydney.edu.au, \{andi.han,junbin.gao\}@sydney.edu.au} \\
}

\begin{document}
\maketitle

\begin{abstract}In this paper, we will present a simple yet effective way for directed Graph (digraph)
Convolutional Neural Networks based on the classic Singular Value Decomposition (SVD),
named SVD-GCN for digraphs. Through empirical experiments on node classification
datasets, we have found that SVD-GCN has remarkable improvements in a number of graph
node learning tasks and outperforms GCN and many other state-of-the-art graph neural
networks.
\end{abstract}

\section{Introduction}

Recent decades have witnessed the sucess of applying deep learning techniques in many domains given that deep learning strategies could effectively extract the latent information from data and efficiently capture the hidden patterns of data, no matter it is textual data or image data \cite{bronstein2017geometric,wu2020comprehensive,zhou2020graph,zhang2020deep,atz2021geometric}. 
While in the recent few years, graph representation learning has emerged and become one
of the most prominent fields in deep learning and machine learning area \cite{battaglia2018relational}. Many real world data are presented or structured in the form of graphs reflecting their intrinsic relations, for example, in a citation network,
documents or papers are represented as nodes while the link between nodes
represents the citation relationship.
Usually the formation and composition of graphs could be irregular and
complicated. To better handle the complexity and capture the hidden
information of graph-structured data, Graph Neural Networks (GNNs) which are
motivated by classical convolutional architectures (CNN) become an
effective and useful way to incorporate relationship information in learning problems or tasks, thus transforming the graph data into low-dimension space while maintaining the
structural information as much as possible \cite{wu2020comprehensive}.

We have seen two major types of methods in approaching GNNs: (1) Spatial-based  approaches and (2) Spectral-based approaches. Spatial methods apply message  passing (or aggregation) from the neighbours of each graph node e.g. their top-k neighbors, while spectral methods use the eigenvectors and eigenvalues of Laplacian matrix (for undirected graphs) with eigen-decomposition, and perform convolutions  with the Graph Fourier Transformation and inverse Graph Fourier transform. Both
techniques were recently brought together \cite{KipfWelling2017}. Though spectral  models have a more solid theoretical foundation in practice, they are less popular in terms of computational capacity than the spatial-based models as the former spectral models rely on eigen-decomposition of the normalized Laplacian matrix from a graph 
while spatial models can directly perform convolutions in the domain of the graph and process a  batch of nodes. Further, information propagation could be done locally and weights can be shared across the graph in different locations easily. 

Much of the recent literature on GNN pays particular attention to explore directed graphs, which are one major type of graph-structured data \cite{KipfWelling2017,HamiltonYingLeskovec2017}. 

Spectral models
are mostly designed for undirected graphs where the graph Laplacian is symmetric positive semi-definite as it provides an orthonormal system for graph Fourier analysis
\cite{wu2020comprehensive}. However, in the cases of directed graphs, the Laplacian matrix is non-symmetric. Its eigen-decomposition
leads to complex pairs of eigenvalues and eigenvectors that cannot form an orthonormal system. Performing convolution on such system will result in incorrect feature aggregation \cite{TongLiangSunLiRosenblumLim2020}. Thus, spatial-based models are usually preferred when it comes to process directed graphs \cite{KipfWelling2017,HamiltonYingLeskovec2017,VelickovicCucurullCasanovaRomeroLioBengio2018}. 

For the spectral methods, in
\cite{TongLiangSunLiRosenblumLim2020}, motivated by the Inception Network
\cite{szegedy2016rethinking}, researchers design and develop scalable receptive fields and avoid those unbalanced receptive fields which are caused by the non-symmetric digraph (directed graph). While Simplifying Graph Convolutional Networks (SGC) removes the nonlinearity between the GCN layers and folds the final function into a final linear model, the SGC outperforms many state-of-the-art GNNs
\cite{WuSouzaZhangFiftyYuWeinberger2019}. Scalable Inception Graph Neural Networks (SIGN), inspired by inception module, is equipped with normalized directed adjacency matrix and its transpose to deal with gigantic directed graphs while still can achieve optimal results with faster speed \cite{RossiFrascaChamberlainEynardBronsteinMonti2020}.

We wish to retain the better performance from spectral GNN methods for directed graphs. To achieve this, different from the existing approaches, we take the advantages of dual orthogonal systems of the adjacency matrix of a directed graph and conduct framelet decomposition over these SVD ``frequencies''.

\subsection{The key idea}

Spectral graph neural network has been proved powerful for graph tasks. 
This is based on the node Laplacian $\Lb$ or the 0-th order Hodge Laplacian shown in  \cite{Lim2020}.
Suppose $\Xb$ is the node signals on graph, then the base operation in spectral GNN is $\Yb$ = $\Lb$ $\Xb$.
For undirected graph, performing SVD of a symmetric matrix is nearly equivalent to EVD
of the matrix, while the only difference is the sign. Consider EVD of the normalized Laplacian
matrix $\widehat{\mathbf L} = \mathbf I - \widehat{\mathbf A}$ where $\widehat{\mathbf A}
= \mathbf {(D+I)}^{-1/2} (\mathbf{A + I}) \mathbf {(D+I)}^{-1/2}$ as $\widehat{\mathbf L}
= \mathbf U \boldsymbol \Lambda \mathbf U^\top = \mathbf U (\mathbf I - \boldsymbol
\Sigma) \mathbf U^\top$, where $\widehat{\mathbf A} = \mathbf U \boldsymbol \Sigma \mathbf
U^\top$. One can verify the eigenvalues of $\widehat{\mathbf L}$ are in $[0,2]$ and
eigenvalues of $\widehat{\mathbf A}$ are in $[-1,1]$. The eigenvalues of Laplacian can be interpreted as the frequencies of graph signals defined on the graph nodes.  

The above standard graph spectral methods usually consider graph Laplacian, which is symmetric, positive semi-definite (for undirected graphs). With these properties, as pointed out in the above, the classic Fourier analysis has been extended with the eigenvectors of Laplacian naturally forming a set of orthonormal basis.

However, this is a different story when looking at directed graphs, as we no longer enjoy the benefit of symmetric
property of the Laplacian. We here provide an alternative strategy based on SVD of adjacency matrix (could be
self-looped)  which plays a key role in 
a typical spatial Graph neural network, relying on the basic operation: $\Yb = \Ab\Xb$, where $\Ab$ is the adjacency matrix which is normally asymmetric for digraphs.

In fact, the adjacency matrix can be regarded as the graph shift operator that replaces the graph signal at one node with the linear combination of its neighbours \cite{8038007,sandryhaila2013discrete}. Regardless of whether the graph is undirected or directed, as long as the adjacency matrix is diagonalizable, we can always factor $\mathbf A = \mathbf V \mathbf \Lambda \mathbf V^{-1}$ and generalize the Fourier transform as $\widehat{\mathbf x} = \mathbf V^{-1} \mathbf x$. The adjacency matrix is diagonalizable for strongly connected directed graph \cite{vanDamOmidi2018} 
and we may consider Jordan decomposition when such condition is violated \cite{sandryhaila2013discrete}. Notice that for general directed graphs, $\mathbf V^{-1} \neq \mathbf V^\top$ and also the $\mathbf S, \mathbf V$ are complex-valued, which poses difficulty to extend classic wavelet/framelet theories.

To avoid such issue, we consider the SVD of the shift operator, $\mathbf A = \mathbf U \mathbf \Lambda \mathbf V^\top$, which provides two sets of orthonormal bases $\mathbf {U, V}$ with real-valued, positive singular values $\mathbf \Lambda$. By applying the shift operator to a graph signal $\mathbf X$, it can be interpreted as first decomposing the signal in terms of the bases defined by the columns of $\mathbf V$, which is followed by a scaling operation defined by $\mathbf \Lambda$. Then the signals are transformed via another set of bases defined by the columns of $\mathbf U$.

As the magnitude of $\mathbf \Lambda$ indeed means the ``frequency'',
we can regulate $\mathbf \Lambda$ by e.g. a modulation function $g$ and define the
following filtering
\begin{align}
\Yb = \sigma((\Vb g(\mathbf \Lambda) \Ub^\top)\cdot g_{\theta} \circ (\Ub g(\mathbf \Lambda) \Vb^\top\Xb) ). \label{eq:1-1}
\end{align}
Extra care should be paid that the smaller singular values means
noised signal components. A well-designed modulation function $g$ shall learn to regulate the ``frequency'' components in the graph signals $\Xb$ in conjunction with a filter $g_{\theta}$, \cite{ChangRongXuHuangSojoudiHuangZhu2021}.



In this paper, we will present a simple yet effective way for directed Graph (digraph) Convolutional Neural Networks based on the classic Singular Value Decomposition (SVD), named SVD-GCN for digraphs. Through empirical experiments on node classification datasets, we have found that SVD-GCN has remarkable improvements in a number of graph node learning tasks and outperforms GCN and many other state-of-the-art graph neural networks.

Our \textit{contributions} in this paper are in four-fold:
\begin{enumerate}
    \item To our best knowledge, this is the first attempt to introduce adjacency SVD for graph  convolutions neural network. To better filter the graph signal in the SVD domain, we apply the graph quasi-framelet decomposition on graph signals so that enhance the performance of the SVD-GCN.
    \item We investigate the way of scaling up the SVD-GCN for large scale graphs based Chebyshev polynomial approximation by deriving fast filtering for singular values without explicitly conducting SVD for the large scale adjacency matrix.
    \item We theoretically prove that the new dual orthogonormal systems offered by SVD provide the guarantees of graph signal decomposition and reconstruction based on the classic and quasi-framelets.
    \item  The results of extensive experiments prove the effectiveness of node representation learning via the framelet SVD-GCNs and show the outperformance against many state-of-the-art methods for digraphs.
\end{enumerate}

\paragraph{Organizations.} The paper is organized as follows. Section~\ref{Sec:2} is dedicated to introducing the relevant works on graph neural networks for directed graphs and reviewing the graph framelets/quasi-framelets developed in the recent years, to pave the way how this can be used for directed graphs. Section~\ref{Sec:3} introduces the theory of graph signal SVD and its combination with the graph framelets/quasi-framelets which leads to our proposed SVD-GCN schemes for directed graphs. In Section~\ref{Sec:4}, comprehensive experiments are conducted to demonstrate the robustness and effectiveness of the proposed SVD-GCN and its performance against a wide range of graph neural networks model/algorithms in node classification tasks. Section~\ref{Sec:5} concludes the paper.


\section{Related Works}\label{Sec:2}
In this section, we will present a summary on the several works regarding GCN for
directed graphs and framelet-based convolutions.

\subsection{GCN for Directed Graphs}
Neural network was first applied to directed acyclic graphs in \cite{sperduti1997supervised}, and this motivated the early studies in GNNs. But these early research mostly focus on Recurrent GNNs (RecGNN) \cite{wu2020comprehensive}. Later, the Graph Neural Network (GNN) concept was proposed and further elaborated in \cite{scarselli2008graph} and \cite{gori2005new}, that expands the application domain of existing neural networks even larger and GNNs can be implemented to process more graph-based data. Inspired by the success of Convolutional Neural Network (CNN) application in computer vision, researchers developed many approaches that can redefine the notion of graph's convolution, and all these approaches and methods all are under umbrella of Convolutional Graph Neural Network (ConvGNN) \cite{wu2020comprehensive}.

We have noted that ConvGNNs can be categorized into the spatial-based and spectral-based approaches.  The majority of them are spatial models, 
the spatial-based approaches utilize neighbour traversal methods to extract and concatenate features, and this in general is implemented via taking adjacency as transformation. 
Researchers did improve models' abilities of features extraction by stacking many graph convolutional layers \cite{wu2020comprehensive}. However, this approach could cause the overfitting problem and feature dilution issue as the models are built deeper with more and more GCN layers \cite{TongLiangSunRosenblumLim2020}.

Recently, more attention has focused on the provision of learning from directed graphs. A novel approach, Directed Graph Convolutional Network (DGCN) propagation model was developed and presented to adapt to digraph \cite{ma2019spectral}. The key idea is to re-define a symmetric so-called normalized Laplacian matrix for digraphs via normalizing and symmetrising the transition probability matrix. 
DGCN does have better performance than the state-of-the-art spectral and spatial GCN approaches on directed graph datasets in semi-supervised nodes classification tasks, but it still has some limitations, such as it increases computational cost, requires very high memory space and it is developed based on the assumption that the input digraph of DGCN is strongly connected \cite{ma2019spectral}.

In \cite{monti2018motifnet}, the researchers proposed a GCN called MotifNet, which is able to deal with directed graphs by exploiting local graph motifs. It basically uses the motif-induced adjacencies, constructed convolution-like graph filters and applied attention mechanism, while the experimental results on the real data shows that MotifNet does have advantages dealing with directed graphs and addressing the drawback of spectral GCNs application in processing graph data, without further increasing the computational cost.

In \cite{TongLiangSunLiRosenblumLim2020}, inspired by the Inception Network module presented in \cite{szegedy2016rethinking}, researchers proposed the Digraph Inception Convolutional Networks (DiGCN) in which they designs and develop scalable receptive fields and avoid those unbalanced receptive fields which are caused by the non-symmetric digraph (directed graph). Through experiments, DiGCN is proved that it's able to learn digraph representation effectively and outperforms mainstream digraph benchmarks’ GCN convolution.

One of common ideas used in all the approaches is to use heuristic to construct and revise the Laplacian matrix. For example, the recent work \cite{zhang2021magnet} proposes a way to define the so-called magnetic Laplacian, as a complex Hermitian matrices, in which the direction information is encoded by complex numbers. 
The experimental results manifest that MagNet's performance exceeds all other approaches on the majority of the tasks such as digraph node classification and link prediction \cite{zhang2021magnet}.




\subsection{Framelets and Quasi-Framelets}
Framelet-based convolutions have been applied to graph neural networks and reveal superior performance in keeping node feature related information and graph geometric information, while framelet-based convolutions also have the advantage of fast algorithm in signal decomposition and reconstruction process \cite{YangZhouYinGao2022}.

Framelet-based convolution for signals defined on manifolds \cite{dong2017sparse} has been recently applied for graph signals \cite{ZhengZhouGaoWangLioLiMontufar2021}. Instead of using a single modulation function $g(\cdot)$ as in \eqref{eq:1-1}, a group of modulation functions in spectral domain were used, i.e., the scaling functions in Framelet terms. This set of functions can jointly regulate the spectral frequency and are normally designed according to the Multiresolution Analysis (MRA) based on a set of (finite) filter bank $\eta = \{a; b^{(1)}, ..., b^{(K)}\} \subset l_0 (\mathbb{Z})$ in spatial domain. Yang et al. \cite{YangZhouYinGao2022} further demonstrate that the MRA is unnecessary, and propose a suffcient idenity condition of a group modulation functions, which is reflected in the following definition

\begin{definition}[Modulation functions for Quasi-Framelets] We call a set of $K+1$ positive modulation functions defined on $[0, \pi]$, $\mathcal{F} = \{g_0(\xi), g_1(\xi), ..., g_K(\xi)\}$, a quasi-framelet if it satisfies the following identity condition
\begin{align}
g_0(\xi)^2 + g_1(\xi)^2 + \cdots + g_K(\xi)^2 \equiv 1,\;\;\; \forall \xi \in [0, \pi] \label{eq:3}
\end{align}
such that $g_0$ decreases from 1 to 0 and $g_K$ increases from 0 to 1 over the spectral domain $[0, \pi]$.
\end{definition}
Particularly $g_0$ aims to regulate the high frequency while $g_K$ to regulate the lower frequency, and the rest to regulate other frequency between. The classic examples include the linear framelet and quadratic framelet functions \cite{dong2017sparse,ZhengZhouGaoWangLioLiMontufar2021}, and sigmoid and entropy quasi-framelet functions \cite{YangZhouYinGao2022}. For the convenience, we list two examples here:

\textit{Linear Framelet Functions} \cite{dong2017sparse}:
\begin{align*}
    g_0(\xi) = \cos^2(\xi/2);\;\;\;\; g_1(\xi)\frac1{\sqrt{2}}\sin(\xi); \;\;\;g_2(\xi) =\sin^2(\xi/2).
\end{align*}

\textit{Entropy Framelet Functions} \cite{YangZhouYinGao2022}:
 \begin{align*}
    g_0(\xi) &= \begin{cases} \sqrt{1 - g^2_1(\xi)}, & \xi <= \pi/2\\
    0, & \text{otherwise}
    \end{cases}\\
        g_1(\xi) &= \sqrt{4\alpha \xi/\pi - 4 \alpha \left(\xi/\pi\right)^2}\\
        g_2(\xi) &= \begin{cases} \sqrt{1 - g^2_1(\xi)}, & \xi > \pi/2\\
    0, & \text{otherwise}
    \end{cases}
    \end{align*}
where $0<\alpha\leq 1$ is a hyperparameter that can be fine-tuned. Note, for $\alpha=1$, $g^2_1(\pi\xi)$ is the so-called \textit{binary entropy function}.

\section{SVD-Framelet Decomposition}\label{Sec:3}
\subsection{The Definition of SVD-Framelets}
The recent work on the undecimated framelets-enhanced graph neural networks (UFG) has enjoyed its great success in many graph learning tasks \cite{ZhengZhouGaoWangLioLiMontufar2021}. As UFG is built upon the spectral graph signal analysis framework by exploiting the power of multiresolution analysis provided by the classic framelet theory \cite{dong2017sparse}, it can only be applied to undirected graphs. On the other side, the classic framelet construction is very restrictive. To explore more meaningful signal frequency decomposition, the authors of \cite{YangZhouYinGao2022}  propose a more flexible way to construct framelets, which are called quasi-framelets.

How to explore such multiresolution lens for digraph signals is a valuable question to ask. This motivates us to look back whether any classic signal analysis methods can be adopted for digraph signals.

Now consider an directed (homogeneous) graph  $\mathcal{G} =(\mathcal{V}, \mathcal{E})$ with $N$ nodes and any graph signal $\mathbf X$ defined on its nodes. Suppose that $\mathbf A\in\mathbb{R}^{N\times N}$ denotes its adjacency matrix which is typically asymmetric and its in-degree and out-degree diagonal matrices $\Db_1$ and $\Db_2$. Now we will consider its self-looped normalized adjacency matrix $\widehat{\Ab} = (\Db_1+\Ib)^{-1/2}(\Ab+\Ib) (\Db_2+\Ib)^{-1/2}$. Typically in the spatial graph neural networks, $\widehat{\Ab}$ is used to define the following convolutional layer
\begin{align}
\Xb' = \widehat{\Ab}\Xb \Wb. \label{Eq:1}
\end{align}
Now suppose we have the following SVD for the normalized (directed) adjacency matrix
\begin{align}
\widehat{\Ab} = \Ub \mathbf \Lambda\Vb^\top, \label{Eq:2}
\end{align}
where $\Ub$ contains the left singular vectors, $\Vb$ contains the right singnular vectors, and $\mathbf \Lambda = \text{diag}(\lambda_1, ..., \lambda_N)$ is the diagonal matrix of all the singular values in decreasing order. Taking \eqref{Eq:2} into \eqref{Eq:1} reveals that we are projecting graph node signals $\Xb$ onto the orthogonal system defined by the columns of $\Vb$, then reconstructing the signals on its \textit{dual} orthogonal system defined by the columns of $\Ub$, with appropriate scaling given by the singular values $\mathbf \Lambda$. This is the place we can ``filter'' graph signals according to the dual orthogonal systems. 

Inspired by the idea of applying undecimated framelets over the Laplacian orthogonal systems (i.e. spectral analysis), we will introduce applying framelets over the dual orthogonal systems defined by SVD.

For a given set of framelet or quasi-framelet functions $\mathcal{F} = \{g_0(\xi), g_1(\xi), ..., g_K(\xi)\}$ defined on $[0, \pi]$\footnote{The reason why we consider this domain as the most classic framelets are defined on $[0, \pi]$. This restriction can be removed for quasi-framelets.}, see \cite{YangZhouYinGao2022,ZhengZhouGaoWangLioLiMontufar2021}. and a given multiresolution level $L$ ($\geq 0$), define the following framelet or quasi-framelet signal decomposition and reconstruction operators
\begin{align}
    \mathcal{W}_{0,L} =& \mathbf V g_0(\frac{\boldsymbol{\Lambda}}{2^{m+L}}) \cdots g_0(\frac{\boldsymbol{\Lambda}}{2^{m}}) \Lambda^{\frac12}\mathbf V^\top, \label{Eq:3}\\
    \mathcal{W}_{k,0} =& \mathbf V g_k(\frac{\boldsymbol{\Lambda}}{2^{m}}) \Lambda^{\frac12}\mathbf V^\top, \text{for } k = 1, ..., K, \label{Eq:4}\\
    \mathcal{W}_{k,\ell} =& \mathbf V g_k(\frac{\boldsymbol{\Lambda}}{2^{m+\ell}})g_0(\frac{\boldsymbol{\Lambda}}{2^{m+\ell-1}}) \cdots g_0(\frac{\boldsymbol{\Lambda}}{2^{m}}) \Lambda^{\frac12}\mathbf V^\top, \label{Eq:5}\\
    &\text{for } k=1, ..., K, \ell = 1, ..., L.\notag
\end{align}
and
\begin{align}
    \mathcal{V}_{0,L} =& \mathbf U \Lambda^{\frac12}g_0(\frac{\boldsymbol{\Lambda}}{2^{m}})   \cdots g_0(\frac{\boldsymbol{\Lambda}}{2^{m+L}}) \mathbf V^\top, \label{Eq:3a}\\
    \mathcal{V}_{k,0} =& \mathbf U \Lambda^{\frac12} g_k(\frac{\boldsymbol{\Lambda}}{2^{m}}) \mathbf V^\top, \text{for } k = 1, ..., K, \label{Eq:4a}\\
    \mathcal{V}_{k,l} =& \mathbf U \Lambda^{\frac12} g_0(\frac{\boldsymbol{\Lambda}}{2^{m}}) \cdots g_0(\frac{\boldsymbol{\Lambda}}{2^{m+\ell-1}}) g_k(\frac{\boldsymbol{\Lambda}}{2^{m+\ell}})\mathbf V^\top, \label{Eq:5a}\\
    & \text{for } k=1, ..., K, \ell = 1, ..., L. \notag
\end{align}
We stack them as $\mathcal{W} = [\mathcal{W}_{0,L}; \mathcal{W}_{1,0}; ...; \mathcal{W}_{K,0}; \mathcal{W}_{1,1}; ...; \mathcal{W}_{K,L}]$ in column direction and $\mathcal{V} = [\mathcal{V}_{0,L}, \mathcal{V}_{1,0}, ..., \mathcal{V}_{K,0}, \mathcal{V}_{1,1}, ..., \mathcal{V}_{K,L}]$ in row direction, then we have
\begin{theorem}\label{Them1} The SVD-GCN layer defined by \eqref{Eq:1} can be implemented by a process of decomposition and reconstruction defined by two operators $\mathcal{W}$ and $\mathcal{V}$, i.e.,
\[
\Xb' = \widehat{\Ab}\Xb\Wb = \mathcal{V}(\mathcal{W}\Xb\Wb).
\]
\end{theorem}
\begin{proof}
Indeed we only need to prove that $\widehat{\Ab} = \mathcal{V}\mathcal{W}$. We will apply the identity property of the (quasi-)framelet functions, i.e., $\sum^K_{k=0} g^2_k(\xi) \equiv 1$. According to the definition of the matrices $\mathcal{W}$ and $\mathcal{V}$, we have
\begin{align*}
&\mathcal{V}\mathcal{W}= \mathcal{V}_{0,L}\mathcal{W}_{0,L} + \sum^L_{\ell=0}\sum^K_{k=1}\mathcal{V}_{k,\ell}\mathcal{W}_{k,\ell} \\
=&\mathcal{V}_{0,L}\mathcal{W}_{0,L} + \sum^K_{k=1}\mathcal{V}_{k,L}\mathcal{W}_{k,L} + \sum^{L-1}_{\ell=0}\sum^K_{k=1}\mathcal{V}_{k,\ell}\mathcal{W}_{k,\ell}\\
=& \mathbf U \Lambda^{\frac12}g_0(\frac{\boldsymbol{\Lambda}}{2^{m}})   \cdots g_0(\frac{\boldsymbol{\Lambda}}{2^{m+L}}) \mathbf V^\top \mathbf V g_0(\frac{\boldsymbol{\Lambda}}{2^{m+L}}) \cdots g_0(\frac{\boldsymbol{\Lambda}}{2^{m}}) \Lambda^{\frac12}\mathbf V^\top \\
&+\sum^K_{k=1}\mathbf U \Lambda^{\frac12} g_0(\frac{\boldsymbol{\Lambda}}{2^{m}}) \cdots g_0(\frac{\boldsymbol{\Lambda}}{2^{m+L-1}}) g_k(\frac{\boldsymbol{\Lambda}}{2^{m+L}})\mathbf V^\top \cdot\\
&\phantom{+\sum^K_{k=1}}\mathbf V g_k(\frac{\boldsymbol{\Lambda}}{2^{m+L}})g_0(\frac{\boldsymbol{\Lambda}}{2^{m+L-1}}) \cdots g_0(\frac{\boldsymbol{\Lambda}}{2^{m}}) \Lambda^{\frac12}\mathbf V^\top + \sum^{L-1}_{\ell=0}\sum^K_{k=1}\mathcal{V}_{k,\ell}\mathcal{W}_{k,\ell}\\
=&\mathbf U \Lambda^{\frac12}g_0(\frac{\boldsymbol{\Lambda}}{2^{m}})   \cdots g_0(\frac{\boldsymbol{\Lambda}}{2^{m+L-1}})\left(\sum^K_{k=0}g^2_k(\frac{\boldsymbol{\Lambda}}{2^{m+L}})\right)\cdot \\
&\phantom{\mathbf U \Lambda^{\frac12}} g_0(\frac{\boldsymbol{\Lambda}}{2^{m+L-1}}) \cdots g_0(\frac{\boldsymbol{\Lambda}}{2^{m}}) \Lambda^{\frac12}\mathbf V^\top + \sum^{L-1}_{\ell=0}\sum^K_{k=1}\mathcal{V}_{k,\ell}\mathcal{W}_{k,\ell}\\
=& \mathcal{V}_{0,L-1}\mathcal{W}_{0,L-1} + \sum^{L-1}_{l=0}\sum^K_{k=1}\mathcal{V}_{k,l}\mathcal{W}_{k,l} \\
=& \cdots = \\
=& \mathcal{V}_{0,0}\mathcal{W}_{0,0} + \sum^K_{k=1} \mathcal{V}_{k,0}\mathcal{W}_{k,0}\\
=& \mathbf U \Lambda^{\frac12}\left( \sum^K_{k=0}  g^2_k(\frac{\boldsymbol{\Lambda}}{2^{m}})\right)\Lambda^{\frac12}\mathbf V =  \mathbf U \Lambda \mathbf V = \widehat{\mathbf A}.
\end{align*}
This completes the proof.
\end{proof}

\subsection{SVD-Framelet Signal Decomposition and Reconstruction}
To explore SVD-GCN layer from graph signal point of view, we define the following graph SVD-framelets.

Suppose $\{(\lambda_{i}, \mathbf u_{i}, \mathbf v_{i})\}^N_{i=1}$  are the singular values and singular vector triples for the noramlized adjacency matrix $\widehat{\mathbf A}$ of graph $\mathcal{G}$ with $N$ nodes, such that $\{\lambda_i\}$ in decreasing order and $\mathbf u_i$ and $\mathbf v_i$ are columns of $\Ub$ and $\Vb$, respectively. For a
graph $\mathcal{G}$, given a set of modulation functions $\mathcal{F} = \{g_0(\xi), g_1(\xi), ..., g_K(\xi)\}$,
the forward SVD-framelets at scale level $\ell = 0,1,...,L$ are defined, for $k = 1,...,K$, by
\begin{align}
\begin{aligned}
\phi_{\ell, p}(q) =\sum^N_{i=1}\sqrt{\lambda_i}g_0(\frac{\lambda_i}{2^{\ell}}) \mathbf u_i (p) \mathbf u_i(q); \\
\psi^k_{\ell, p}(q) =\sum^N_{i=1}\sqrt{\lambda_i}g_k(\frac{\lambda_i}{2^{\ell}}) \mathbf v_i (p) \mathbf v_i(q)
\end{aligned}
\end{align}
and its corresponding backward SVD-framelets by
\begin{align}
\begin{aligned}
\overline{\phi}_{\ell, p}(q) =\sum^N_{i=1}\sqrt{\lambda_i}g_0(\frac{\lambda_i}{2^{\ell}}) \mathbf u_i (p) \mathbf v_i(q); \\
\overline{\psi}^k_{\ell, p}(q) =\sum^N_{i=1}\sqrt{\lambda_i}g_k(\frac{\lambda_i}{2^{\ell}}) \mathbf u_i (p) \mathbf v_i(q)
\end{aligned}
\end{align}
for all nodes $q, p$ and $\phi_{\ell, p}$ ($\overline{\phi}_{\ell, p}$) or $\psi^k_{\ell, p}$ ($\overline{\psi}_{\ell, p}$) are the low-pass or high-pass SVD-framelet translated at node $p$.

Similar to the standard undecimated framelet system \cite{dong2017sparse}, we can define two SVD-framelet systems as follows:
\begin{align}
    \text{SVD-UFS-F}_L(\mathcal{F}; \mathcal{G}):=& \{\phi_{0, p}:p\in\mathcal{V}\} \cup \notag\\
    &\{\psi^k_{\ell,p}: p\in\mathcal{V}, \ell = 0, ..., L\}^K_{k=1}, \label{Eq:new1}\\
    \text{SVD-UFS-B}_L(\mathcal{F}; \mathcal{G}):=& \{\overline{\phi}_{0, p}:p\in\mathcal{V}\} \cup \notag\\
    &\{\overline{\psi}^k_{\ell,p}: p\in\mathcal{V}, \ell = 0, ..., L\}^K_{k=1}. \label{Eq:new2}
\end{align}

Then the signal transform $\xb' = \widehat{\Ab}\xb$ can be implemented via the SVD-Framelet systems as shown in the following theorem

\begin{theorem}[SVD-Framelet Transform]\label{Them2} Given the above definition of both forward and backward SVD-framelet systems, we have
\begin{align}
\mathbf x' =  \sum_{p\in\mathcal{V}}  \langle \phi_{L, p}, \mathbf x\rangle \overline{\phi}_{L, p} + \sum^K_{k=1}\sum^L_{\ell=0} \sum_{p\in\mathcal{V}} \langle \psi^k_{\ell, p}, \mathbf x\rangle \overline{\psi}_{\ell, p}.
\end{align}
\end{theorem}
\begin{proof}
The decomposition here is indeed the re-writting of $\mathbf x' = \widehat{A}\mathbf x$ in terms of columns of all the decomposition and reconstruction operators $\mathcal{W}$ and $\mathcal{V}$ under the modulation identity condition \eqref{eq:3}.
\end{proof}

This is to say that the transformed signal $\mathbf x'$ can be written as the linear combination of the backward SVD-framelet system with coefficients on the original signal on the forward SVD-framelet system. Thus the signal filtering can be conducted by filtering over the forward SVD-framelet coefficients $\langle \psi^k_{\ell, p}, \mathbf x\rangle$.

%
\begin{figure*}[htp]
    \centering
    \includegraphics[width=0.9\textwidth]{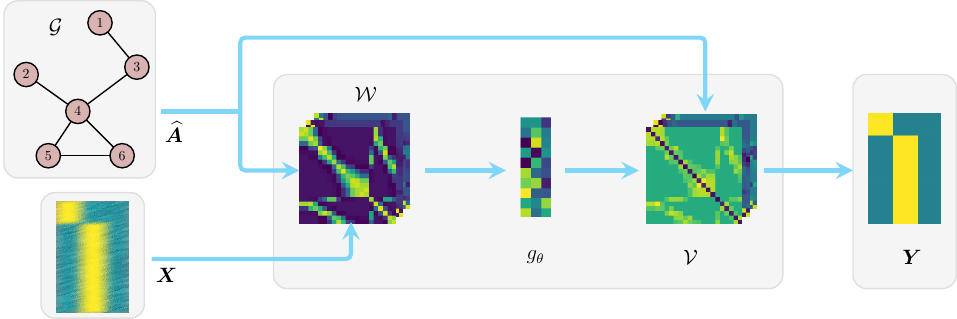}
    \caption{SVD Framelet Layer converts the input node feature $\mathbf X$ by using SVD framelet matrices $\mathcal{W}$ and $\mathcal{V}$ along with learnable filters $g_{\theta}$ to the new features $\mathbf Y$, as demonstrated in the simplified versions \eqref{Eq:9} and \eqref{Eq:10}.}
    \label{Fig:1}
\end{figure*}

\subsection{Simplified SVD-Framelet Filtering and the Model Architecture}
Based on Theorem~\ref{Them1} and \ref{Them2}, for a demonstration, we define the following (simplified Level-1: corresponding to $L=0$) SVD-Framelet filtering
\begin{align}
\Yb = \sigma\left(\sum^K_{k=0}(\Ub  \Lambda^{\frac12}g_k(\Lambda) \Vb^\top)\cdot g^i_{\theta} \circ (\Vb g_k(\Lambda)  \Lambda^{\frac12}\Vb^\top \Xb \Wb) \right)  \label{Eq:9}
\end{align}
where $g^k_{\theta}$'s are the filters corresponding to each modulation function $g_k$ individually, $\Wb$ is learnable feature transformation weight and $\sigma$ is an activation function.

It is not necessary to transform the signal into the backward SVD-framelet space. We may consider the following signal transformation within the forward SVD-framelet space as the following way
\begin{align}
\Yb = \sigma\left(\sum^K_{k=0}(\Vb  \Lambda^{\frac12}g_k(\Lambda) \Vb^\top)\cdot g^k_{\theta} \circ (\Vb g_k(\Lambda)  \Lambda^{\frac12}\Vb^\top \Xb \Wb) \right) \label{Eq:10}
\end{align}
in which $\Ub$ is replaced by $\Vb$.

It is not hard to write their counterparts for level 2. We call scheme \eqref{Eq:9} SVD-Framelet-I (including multiple levels) and \eqref{Eq:10} (including multiple levels) SVD-Framelet-II, however in the experiment part, we mainly focus on SVD-Framelet-I.  

Figure~\ref{Fig:1} shows the information flow for one SVD framelet layer. First the SVD will be conducted on the graph (normalized) adjacency matrix $\widehat{\mathbf A}$ to produce all the framelet matrices $\mathcal{W}$ (primary) and $\mathcal{V}$ (dual) up to a given level $L$,  then the primary framelet matrices $\mathcal{W}$ will be applied to the input node signal matrix $\mathbf X$, followed by the learnable filters on each node across all the channels, then the dual framelet matrices $\mathcal{V}$ will bring the signal back to the transformed signal domain, i.e., $\mathbf X'$, which will be pipelined to the next layers or downstream tasks.  When multiple SVD layers are used in the final architecture, all the primary and dual framelet matrices are shared through all the SVD layers.


\subsection{Faster filtering for large graphs based on Chebyshev polynomials}\label{subsec:3.4}
For large graphs, performing SVD on the adjacency matrix can be costly. We thus consider an approximated filtering based on Chebyshev polynomials. We will take another strategy to realize the idea presented in 
\cite{onuki2017fast} and to derive the fast filtering for singular values.  Suppose that for the normalized adjacency matrix $\widehat{\mathbf A} = \Ub\Lambda \Vb^\top = \widehat{\mathbf A}\Vb \Vb^\top$ and $\widehat{\mathbf A}^\top \widehat{\mathbf A} = \Vb \Lambda^2 \Vb^\top$. In other words, the columns of $\Vb$ gives the eigenvector systems of $\widehat{\mathbf A}^\top \widehat{\mathbf A}$ for which we can conduct framelet analysis as done for Laplacian matrix.

For a given set of framelet or quasi-framelet functions $\mathcal{F} = \{g_0(\xi), g_1(\xi), ..., g_K(\xi)\}$ defined on $[0, \pi]$, as for \eqref{Eq:3} - \eqref{Eq:5}, we define the following framelet or quasi-framelet signal decomposition operators (without confusion we use the same notation)
\begin{align}
    \mathcal{W}_{0,L} =& \mathbf V g_0(\frac{\boldsymbol{\Lambda}^2}{2^{L+m}}) \cdots g_0(\frac{\boldsymbol{\Lambda}^2}{2^{m}})  \mathbf V^\top, \label{Eq:3b}\\
    \mathcal{W}_{k,0} =& \mathbf V g_k(\frac{\boldsymbol{\Lambda}^2}{2^{m}})  \mathbf V^\top, \text{for } k = 1, ..., K, \label{Eq:4b}\\
    \mathcal{W}_{k,\ell} =& \mathbf V g_k(\frac{\boldsymbol{\Lambda}^2}{2^{m+\ell}})g_0(\frac{\boldsymbol{\Lambda}^2}{2^{m+\ell-1}}) \cdots g_0(\frac{\boldsymbol{\Lambda}^2}{2^{m}})  \mathbf V^\top, \label{Eq:5b}\\
    &\text{for } k=1, ..., K, \ell = 1, ..., L.\notag
\end{align}
Note that we have $\boldsymbol{\Lambda}^2$ inside all $g$'s but no the extra term $\boldsymbol{\Lambda}^{\frac12}$.

To avoid any explicit SVD decomposition for $\Vb$, we consider a polynomial approximation to each modulation function $g_j(\xi)$ ($j=0, 1, ..., K$). We  approximate $g_j(\xi)$ by Chebyshev polynomials $\mathcal{T}^n_j(\xi)$ of a fixed degree $n$ where the integer $n$ is chosen such that the Chebyshev polynomial approximation is of high precision. For simple notation, in the sequel, we use  $\mathcal{T}_j(\xi)$ instead of $\mathcal{T}^n_j(\xi)$. Then the new SVD-framelet transformation matrices defined in \eqref{Eq:3b} - \eqref{Eq:5b} can be approximated by, {for } $k=1, ..., K, \ell = 1, ..., L$,
\begin{align}
    \mathcal{W}_{0,L} &\approx  \mathcal{T}_0(\frac1{2^{L+m}}\widehat{\mathbf A}^\top \widehat{\mathbf A}) \cdots \mathcal{T}_0(\frac{1}{2^{m}}\widehat{\mathbf A}^\top \widehat{\mathbf A}), \label{eq:Ta}\\
    \mathcal{W}_{k,0} &\approx    \mathcal{T}_k(\frac1{2^{m}}\widehat{\mathbf A}^\top \widehat{\mathbf A}),  \label{eq:Tb}\\
    \mathcal{W}_{k,\ell} &\approx \mathcal{T}_k(\frac{1}{2^{m+\ell}}\widehat{\mathbf A}^\top \widehat{\mathbf A})\mathcal{T}_0(\frac{1}{2^{m+\ell-1}}\widehat{\mathbf A}^\top \widehat{\mathbf A}) \cdots \mathcal{T}_0(\frac{1}{2^{m}}\widehat{\mathbf A}^\top \widehat{\mathbf A}), \label{eq:Tc}
\end{align}
Then there is no need for SVD of the adjacency matrix to calculate all the framelet matrices. So a new simplified one scale level SVD-Framelet-III (corresponding to $L=0$) can be defined as
\begin{align}
\Yb = \sigma\left(\widehat{\Ab}\sum^K_{k=0}\mathcal{W}^\top_{k,0}\cdot g^k_{\theta}\circ (\mathcal{W}_{k,0}\Xb \Wb)\right) \label{Eq:11}
\end{align}
Similarly the version of multiple levels can be easily written out.



\section{Experiments}\label{Sec:4}



Our main purpose is to demonstrate the proposed SVD-GCN is powerful in assisting graph learning. We will
evaluate our model in various graph learning tasks, including node classification, graph feature denoising, and applications to larger scale graph data. Since our focus is on the directed graph learning tasks, other spectral methods such as UFG \cite{ZhengZhouGaoWangLioLiMontufar2021} and QUFG \cite{YangZhouYinGao2022}  are not considered.  The experiment code can be found at \url{https://github.com/ThisIsForReview/SVD-GCN} for review.

\subsection{Datasets and Baselines}\mbox{}
\indent\textbf{Datasets} We utilize several digraph datasets from Python package \texttt{Torch\_Geometric} \url{https://pytorch-geometric.readthedocs.io/} 
datasets in the experiments including: \texttt{cora\_ml}, \texttt{citeseer}, and \texttt{citeseer\_full} which are citation networks, as well as the Amazon Computers and Amazon Photo co-purchase networks: \texttt{amazon\_photo} and \texttt{amazon\_cs}, see \url{https://github.com/EdisonLeeeee/GraphData}.  

\begin{table}
\caption{Statistics of the datasets}
    \label{Table1}
    \centering
    \begin{tabular}{c|c|c|c|c}\hline
        Datasets & \#Nodes & \#Edges & \#Classes & \#Features   \\ \hline
        cora\_ml \cite{bojchevski2017deep} & 2,995 & 8,416 & 7 & 2,879\\
        citeseer \cite{yang2016revisiting} & 3,312 & 4,715 & 6 & 3703\\
        citeseer\_full \cite{chen2018fastgcn} & 3,327 & 3,703 & 6 & 602\\
        amazon\_photo \cite{shchur2018pitfalls} & 7,650 & 143,663 & 8 & 745\\
        amazon\_cs \cite{shchur2018pitfalls} & 13,752 & 287,209 & 10 & 767\\
        \hline
        coral\_full  \cite{BojchevskiGuennemann2018} & 19,793 & 65,311 & 70 & 8,710 \\ \hline
    \end{tabular}
\end{table}

Here is the brief descriptions about the datasets we use to conduct the experiments and Table 1 summarizes the statistics of five datasets:
\begin{itemize}
\item \texttt{cora\_ml} \cite{bojchevski2017deep}: Cora is a classic citation network dataset while Cora\_ml is a small subset of dataset that \cite{bojchevski2017deep} extracted from the entire original network Cora dataset and cora\_ml is also a directed network dataset, which means that the edge between all pairs of nodes is directed, ie. A pointed to B means that A cited B.

\item \texttt{citeseer} \cite{yang2016revisiting} \& \texttt{citeseer\_full}\cite{chen2018fastgcn}: Citeseer is also a citation network dataset whose nodes represent documents and paper while edges are citation links. Citeseer\_full is extracted from the same original network dataset as citeseer that nodes represent documents and edges represent citation links. The only difference is that dataset split type is full, which means that except the nodes in the validation and test sets, all the rest of nodes are used in training set.

\item \texttt{amazon\_photo} \& \texttt{amazon\_cs} \cite{shchur2018pitfalls}: The Amazon Computers and Amazon Photo network datasets are both extracted from Amazon co-purchase Networks. The nodes represents the goods and edges represent that the two nodes(goods) connected are frequently bought together, and the product reviews are bag-of-words node features.
\end{itemize}

\begin{table*}[ht]
\caption{Results for Classification Accuracy (\%): Part Results from \cite{TongLiangSunLiRosenblumLim2020}}
    \label{Table2}
    \centering
    \begin{tabular}{c|c|c|c|c|c}\hline
        Models & cora\_ml & citeseer & citeseer\_full & amazon\_photo & amazon\_cs\\ \hline
        ChebNet \cite{DefferrardBressonVandergheynst2016} & \cellcolor{white!50!blue!20!}64.02$\pm$1.5 & 56.46$\pm$1.4 & 62.29$\pm$0.3 & \cellcolor{white!50!blue!40!}80.91$\pm$1.0& \cellcolor{white!50!blue!45!}73.25$\pm$0.8\\
        GCN \cite{KipfWelling2017}& 53.11$\pm$0.8 & 54.36$\pm$0.5 & 64.71$\pm$0.5  & 53.20$\pm$0.4 & 60.50$\pm$1.6\\
        SGC \cite{WuSouzaZhangFiftyYuWeinberger2019} & 51.14$\pm$0.6 & 44.07$\pm$3.5 & 56.56$\pm$0.4 & 71.25$\pm$1.3 & \cellcolor{white!50!blue!55!}76.17$\pm$0.1\\
        APPNP \cite{KlicperaBojchevskiGuennemann2019} & \cellcolor{white!50!blue!30!}70.07$\pm$1.1 & \cellcolor{white!50!blue!70!}65.39$\pm$0.9 & \cellcolor{white!50!blue!40!}67.53$\pm$0.4 & \cellcolor{white!50!blue!35!}79.37$\pm$0.9 & \cellcolor{white!50!blue!20!}63.16$\pm$1.4\\
        InfoMax \cite{VelickovicFedusWilliamHamiltonLioBengioDevonHjelm2019} & 58.00$\pm$2.4 &  \cellcolor{white!50!blue!20!}60.51$\pm$1.7 & \cellcolor{white!50!blue!55!}72.93$\pm$1.1& \cellcolor{white!50!blue!25!}74.40$\pm$1.2 & 47.32$\pm$0.7\\
        GraphSAGE \cite{HamiltonYingLeskovec2017} & \cellcolor{white!50!blue!45!}72.06$\pm$0.9 & \cellcolor{white!50!blue!45!}63.19$\pm$0.7 & \cellcolor{white!50!blue!20!}65.18$\pm$0.8  & \cellcolor{white!50!blue!58!}87.57$\pm$0.9 & \cellcolor{white!50!blue!65!}79.29$\pm$1.3\\
        GAT \cite{VelickovicCucurullCasanovaRomeroLioBengio2018} & \cellcolor{white!50!blue!38!}71.91$\pm$0.9 & \cellcolor{white!50!blue!40!}63.03$\pm$0.6  & \cellcolor{white!50!blue!30!}66.67$\pm$0.4 &  \cellcolor{white!50!blue!80!}89.10$\pm$0.7 & \cellcolor{white!50!blue!70!}79.45$\pm$1.5\\
        DGCN \cite{TongLiangSunRosenblumLim2020} & \cellcolor{white!50!blue!45!}75.02$\pm$0.5 & \cellcolor{white!50!blue!80!}{66.00}$\pm${0.4} & \cellcolor{white!50!blue!85!}78.35$\pm$0.3 & \cellcolor{white!50!blue!50!}83.66$\pm$0.8 & OOM\\
        SIGN \cite{RossiFrascaChamberlainEynardBronsteinMonti2020} & \cellcolor{white!50!blue!25!}64.47$\pm$0.9 & \cellcolor{white!50!blue!30!}60.69$\pm$0.4 & \cellcolor{white!50!blue!77!}77.44$\pm$0.1 & \cellcolor{white!50!blue!20!}74.13$\pm$1.0 & \cellcolor{white!50!blue!35!}69.40$\pm$4.8\\
        DiGCN-PR \cite{TongLiangSunLiRosenblumLim2020}& \cellcolor{white!50!blue!60!}77.11$\pm$0.5 & \cellcolor{white!50!blue!55!}64.77$\pm$0.6 & \cellcolor{white!50!blue!65!}74.18$\pm$0.7 &  OOM & OOM\\
        DiGCN-APPR \cite{TongLiangSunLiRosenblumLim2020}& \cellcolor{white!50!blue!55!}77.01$\pm$0.4 & \cellcolor{white!50!blue!60!}64.92$\pm$0.3 &  \cellcolor{white!50!blue!70!}74.52$\pm$0.4 &   \cellcolor{white!50!blue!65!}88.72$\pm$0.3 &  \cellcolor{white!50!blue!90!}85.55$\pm$0.4\\
        \textbf{SVD-GCN (Ours)} & \cellcolor{white!50!blue!70!}\textbf{78.84}$\pm$\textbf{0.29} &\cellcolor{white!50!blue!95!}\textbf{66.15}$\pm$\textbf{0.39}$^*$ &\cellcolor{white!50!blue!95!}\textbf{80.95}$\pm$\textbf{0.36} & \cellcolor{white!50!blue!70!}\textbf{88.76}$\pm$\textbf{0.21} &\cellcolor{white!50!blue!90!}\textbf{85.55}$\pm$\textbf{0.31} \\ \hline
        DiGCN-APPR-IB \cite{TongLiangSunLiRosenblumLim2020}& \cellcolor{white!50!blue!90!}80.25$\pm$0.5 & \cellcolor{white!50!blue!85!}{66.11}$\pm${0.7} & \cellcolor{white!50!blue!90!}80.10$\pm$0.3 & \cellcolor{white!50!blue!95!}\textbf{90.02}$\pm$\textbf{0.5} & \cellcolor{white!50!blue!100!}\textbf{85.94}$\pm$\textbf{0.5}\\
        \textbf{SVD-GCN-IB (Ours)} & \cellcolor{white!50!blue!100!}\textbf{81.11}$\pm$\textbf{0.24} & \cellcolor{white!50!blue!50!}64.26$\pm$0.77 & \cellcolor{white!50!blue!100!}\textbf{83.12}$\pm$\textbf{0.68} & \cellcolor{white!50!blue!80!}89.38$\pm$0.48 & \cellcolor{white!50!blue!85!}85.03$\pm$0.37 \\ \hline
    \end{tabular}
\end{table*}

\textbf{Baseline Models} We will compare our model to twelve state-of-the-art models including: spectral-based GNNs such as ChebNet \cite{DefferrardBressonVandergheynst2016}, GCN \cite{KipfWelling2017}, APPNP \cite{KlicperaBojchevskiGuennemann2019}, SGC \cite{WuSouzaZhangFiftyYuWeinberger2019} and InfoMax \cite{VelickovicFedusWilliamHamiltonLioBengioDevonHjelm2019}; spatial-based GNNs having GraphSAGE \cite{HamiltonYingLeskovec2017} and GAT \cite{VelickovicCucurullCasanovaRomeroLioBengio2018}; Graph Inception including SIGN \cite{RossiFrascaChamberlainEynardBronsteinMonti2020}; Digraph GNNs containing DGCN; Digraph Inception having DiGCN-PR \cite{TongLiangSunLiRosenblumLim2020}, DiGCN-APPR \cite{TongLiangSunLiRosenblumLim2020} and DiGCN-APPR-IB \cite{TongLiangSunLiRosenblumLim2020}.

At the time of preparing the draft, we are not aware the recently published paper \cite{zhang2021magnet} on MagNet convolutional network for directed graphs. Thus we do not include here for comparison. Nevertheless, we highlight that the number of parameters is double the size compared to our model and thus results on \cite{zhang2021magnet} are not comparable. Further, the standard deviation of MagNet is much larger compared to our model, indicating less model robustness.

\textbf{Setup}  We design our SVD-Framelet-I model \eqref{Eq:9} with one convolution layer plus a fully connected linear layer for learning the graph embedding, the output of which is proceeded by a softmax activation for final prediction. Most hyperparameters are set to default in our program, except for learning rate, weight decay, and hidden units in the layers. We conduct a grid search for fine tuning on these hyperparameters from the pre-defined search space. All the models including those compared models that we need re-conduct are trained with the ADAM optimizer. The maximum number of epochs is basically 200.

\textbf{Hardware} Most of experiments run in PyTorch 1.6 on NVIDIA® Tesla V100 GPU with 5,120 CUDA cores and 16GB HBM2 mounted on an HPC cluster and some experiments run with PyTorch 1.8 on PC with Intel(R) Core(TM) i5-8350U CPU@1.90GHz with 64 bits Window 10 operating system and 16GB RAM. 

\subsection{Graph Node Classification}\label{subsec:4.2}
We conduct our experiments on dataset \texttt{cora\_ml}, \texttt{citeseer}, \texttt{cites} \texttt{eer\_full}, \texttt{amazon\_cs} and \texttt{amazon\_photo}.  In this set of experiments, the hyperparameters are searched as the following ways: 20 nodes per classes are randomly selected for training with 500 nodes as validation and the rest in testing; the basic epoch = 200 (but for \texttt{citeseer} it is 500); we use the two level framelets (corresponding to $L=1$), the scale in framelets (i.e. $s$ in $s^l$ to replace $2^l$) is tested for 1.1, 1.5 and 2.0; the number of hidden features = 16, or 32, or 64; the dropout ratio = 0.1, 0.3, or 0.6; the framelet modulation function is either \texttt{Linear} or \texttt{Entropy} with hyperparameter $\alpha = 0.1, 0.3, 0.5, 0.7$ and $0.9$. We find that both \texttt{Linear} and \texttt{Entropy} are comparable, thus the reported results are based on the \texttt{Linear} framelets. The network architecture consists of one SVD framelet layer (SVD-Framelet-I or SVD-Framelet-II for \texttt{citeseer}) and a fully connected linear layer to the softmax output layer.


The whole experimental results are summarized in Table~\ref{Table2}. Except for \texttt{citeseer\_full}, we copied the results from \cite{TongLiangSunLiRosenblumLim2020} for convenience. The darker blue colour of the result cell represents the higher accuracy rate the approach generates using that dataset (within that column). Compared with all the other state-of-the-art baseline methods, the proposed method obtains the best performance or comparable performance, see all the rows above the two bottom rows.   For datasets \texttt{cora\_ml}, \texttt{citeseer} and \texttt{citeseer\_full}, more than 1\% gains have been obtained by such as a simply convolution layer. For two Amazon datasets, the SVD-GCN is comparable the most state-of-the-art model DiGCN-APPR. For \texttt{citeseer} data, we use the simplified model SVD-Frameelet-II \eqref{Eq:10}. We also note that for the \texttt{citeseer\_full} dataset, the SVD-GCN has a gain of more than 6\% accuracy than the state-of-the-art DiGCN-PR and DiGCN-APPR \cite{TongLiangSunLiRosenblumLim2020}.

The authors of \cite{TongLiangSunLiRosenblumLim2020} further present the Digraph Inception Convolutional Networks (DiGCN-APPR-IB) which utilizes digraph convolution and $k$th-order proximity to achieve larger receptive fields and learn multi-scale features in digraphs. We also adapt their strategy and combine their architecture in our SVD-Framelet-I framelet, resulting in SVD-GCN-IB. The combined model achieves the similar performance as the original DiGCN-APPR-IB, see the last two rows in Table~\ref{Table2}. The results demonstates that increasing receptive fields does improve the performance of both DiGCN and SVD-GCN.

\subsection{Large Scale Experiment}
In this experiment, we wish to demonstrate that the fast SVD-Framelet-III introduced in Section~\ref{subsec:3.4}. All the experiments will be conducted on the dataset \texttt{cora\_full}, see \url{https://github.com/EdisonLeeeee/GraphData}. \texttt{cora\_full}, introduced in \cite{BojchevskiGuennemann2018}, is the full extension of the smaller citation network dataset \texttt{cora\_ml} and it consists of 19,793 nodes and 65,311 edges with node feature dimension 8,710. The number of node classes is 70. This is a quite large graph for most graph neural networks. In fact, for graphs beyond 15K nodes we had to revert to slow training on the CPU since the data did not fit on the GPU memory (12GB).

To get a sense on how reliable and robust our simplified fast SVD-Framelet-III based on the Chebyshev polynomial approximation is, we conduct the experiment only against the manageable benchmark GCN. In fact, for many state-of-the-art algorithm like DiGCN-PR or DiGCN-APPR, we failed to make the program run on the CPU.

Our experimental results are reported in Table~\ref{Table4}.  We follow the similar setting as the first set of experiments in Section~\ref{subsec:4.2}: we choose 20 nodes per class (in total 70 classes), 500 random nodes for validation and the rest are used as testing nodes. This time, we fixed the framelet scale to $s=1.1$ with the \texttt{Linear} framelet modulation functions and the dropout ratio to 0.1 with activation function \texttt{relu} applied to the output from the SVD framelet layer. The overall network archictecture consists of one SVD framelet layer and a fully connected linear layer to the output softmax layer. As the dimension of the node feature is almost 8,710, we tested three different choices of the hidden unit size from 64, 128 and 256. Each experiment was run 10 replicates and each replicate runs 200 epochs with a fixed learning rate 0.005. We report the average test accuracy and their std. The experiment shows that the number of hidden unit 128 is appropriate for this dataset and that the SVD-GCN has performance gain of 1 - 3\% more in all the cases.
\begin{table}
\caption{Results between SVD-Framelet-III and GCN over \texttt{cora\_full}}
    \label{Table4} 
    \centering
    \begin{tabular}{c|c|c}\hline
       NoHiddenUnits & GCN & \textbf{SVD-GCN (Ours)}\\ \hline
      64 & 52.54$\pm$0.41 & \textbf{53.52}$\pm$\textbf{0.33}   \\    
       128 & 53.75$\pm$0.17   & \textbf{56.17}$\pm$\textbf{0.16}   \\ 
       256  & 54.33$\pm$0.33   & \textbf{57.30}$\pm$\textbf{0.21}    \\ \hline  
    \end{tabular}
\end{table}

\subsection{Denoising Capability and Robustness}
We conduct a denoising experiment to assess the robustness of the proposed SVD-Framelet GCNs against the state-of-the-art method DiGCN-APPR. The experiments are conducted on the benchmark dataset \texttt{cora\_ml} for convenience. The data features used in \cite{TongLiangSunLiRosenblumLim2020} is normalized data with values ranged between 0.0 and 1.0. For the purpose of testing denoising capability of the two models SVD-Framelet-I and DiGCN-APPR, we randomly inject Gaussian noises of mean 0 and std (noise levels) ranging from 0.01 to 5. We report the results for noise levels 0.01, 0.05, 0.1, 0.5, 1.0, and 5.0 in Table~\ref{Table3}. We did not report the results for DiGCN-APPR for larger noise levels due to its poor figures. The darker color in the Table 4 represents the higher accuracy rate which further demonstrates the better denoising capability as it still can maintain a relatively high accuracy under different degrees of noise attack. If the accuracy rate is smaller than 40\%, then the result cell will not be colored since that result is quite low and not comparable. From Table 4, it is obvious that when the noise level is larger than 0.01, the DiGCN-APPR method's accuracy rate has dropped dramatically, meanwhile SVD-GCN method still maintain a relatively high accuracy under the noise level of 5.0, and the result is even higher than DiGCN-APPR method's accuracy at noise level of 0.01.

\textbf{Sensitivity analysis:} It is evident that DiGCN-APPR almost fails in denoising data, or in other words, it is quite easy to be attacked by noises. For example, the test accuracy degrades almost 24\% from the case of noise free to the case of noise level 0.01.  However, SVD-GCN has better denoising capalibity, which benefits from the framelet decomposition over the SVD ``frequency'' domain and being filtered in its learning process. The experimental results demonstrate SVD-GCN's performance consistency and robustness to the larger noise levels. From Figure~\ref{Fig:2}, we can observe that, at noise level of 50, the test accuracy is still above 50\% with acceptable standard deviation, which further proves its robustness and stability when encountering much larger noise attack.

\begin{figure}[htp]
    \centering
    \includegraphics[width=0.45\textwidth]{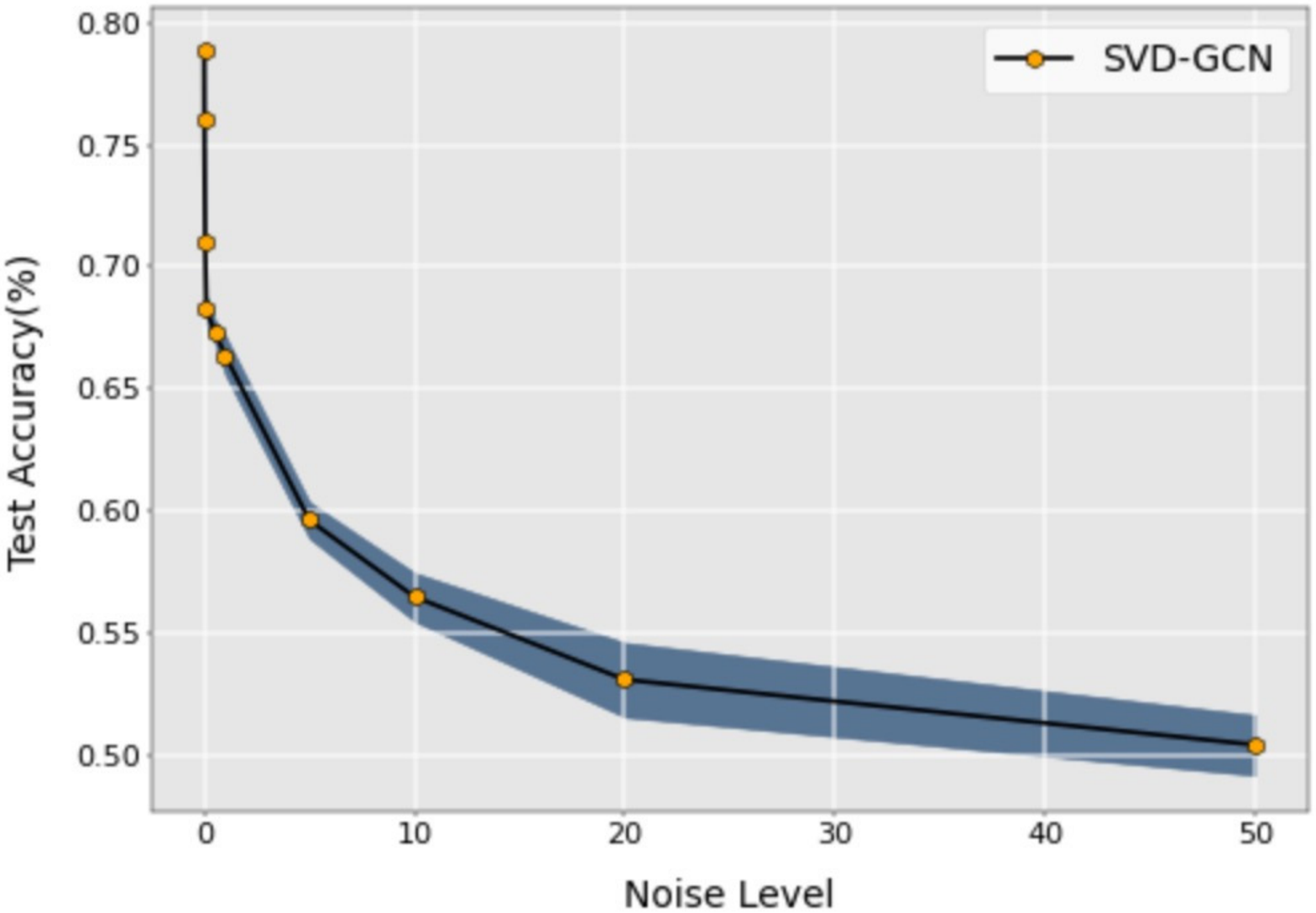}
    \caption{Node attribute perturbation analysis on \texttt{cora\_ml} dataset}
    \label{Fig:2}
\end{figure}

\begin{table}
\caption{Denoising Capability Comparions between SVD-Framelet-I and DiGCN-APPR over \texttt{cora\_ml}}
    \label{Table3}
    \centering
    \begin{tabular}{c|c|c}\hline
       NoiseLevel & DiGCN-APPR \cite{TongLiangSunLiRosenblumLim2020} & \textbf{SVD-GCN (ours)}\\ \hline
       $\sigma=0.0$ & \cellcolor{white!50!blue!80!}77.01$\pm$0.40 & \cellcolor{white!50!blue!85!}78.84$\pm$0.29\\
       $\sigma=0.01$ & \cellcolor{white!50!blue!45!}53.39$\pm$0.61 &  \cellcolor{white!50!blue!78!}76.04$\pm$0.51\\
       $\sigma=0.05$ & 35.72$\pm$0.80  &  \cellcolor{white!50!blue!70!}71.04$\pm$0.50\\
       $\sigma=0.1$ & 34.08$\pm$1.34  &  \cellcolor{white!50!blue!65!}68.25$\pm$0.57\\
       $\sigma=0.5$ & 30.40$\pm$1.92  & \cellcolor{white!50!blue!61!}67.37$\pm$0.47\\
       $\sigma=1.0$ & --- & \cellcolor{white!50!blue!58!}66.38$\pm$0.80\\
       $\sigma=5.0$ & --- & \cellcolor{white!50!blue!53!}59.63$\pm$0.74\\ \hline
    \end{tabular}
\end{table}






\section{Conclusions}\label{Sec:5}
In this paper, we explored the application of framelets over the dual orthogonal systems defined by singular vectors from singular value decomposition (SVD) for graph data and thus proposed a simple yet effective SVD-GCN for directed graphs. The successful improvement from SVD-GCN benefits from 
application of the graph SVD-framelets in  transforming and filetering directed graph signals. The experimental results manifest that the SVD-GCN outperforms all the baselines and the state-of-the-arts on the five commonly used benchmark datasets, which further proves that the proposed SVD-GCN's performance in directed graph learning tasks is remarkable.
The sensitivity analysis also demonstrates that this novel approach has strong denoising capability and it is robust to high level noising attack to node features. 
It has also been proved by experiments that the new fast SVD-GCN is convincingly accurate and reliable appropriate in dealing with large scale graph datasets. 

\bibliographystyle{plain}

\end{document}